\documentclass{article}
\usepackage[margin=1in]{geometry}

\usepackage{xcolor}
\usepackage{hyperref}
\hypersetup{
    colorlinks,
    linkcolor={red!50!black},
    citecolor={blue!50!black},
    urlcolor={blue!80!black}
}
\usepackage{ifthen}
\usepackage{amsthm}
\usepackage{amsmath,amssymb,amsfonts,bbm,dsfont}
\usepackage{caption}
\usepackage{censor}
\usepackage{booktabs}
\usepackage{nicefrac}      
\usepackage{microtype}
\usepackage{afterpage}
\usepackage{inconsolata}

\usepackage{xcolor}
\usepackage{placeins}
\usepackage{pgfplots}
\usepackage{tikz}
\usetikzlibrary{patterns,patterns.meta}
\pgfdeclarepatternformonly{south east lines}{\pgfqpoint{-0pt}{-0pt}}{\pgfqpoint{3pt}{3pt}}{\pgfqpoint{3pt}{3pt}}{
    \pgfsetlinewidth{0.4pt}
    \pgfpathmoveto{\pgfqpoint{0pt}{3pt}}
    \pgfpathlineto{\pgfqpoint{3pt}{0pt}}
    \pgfpathmoveto{\pgfqpoint{.2pt}{-.2pt}}
    \pgfpathlineto{\pgfqpoint{-.2pt}{.2pt}}
    \pgfpathmoveto{\pgfqpoint{3.2pt}{2.8pt}}
    \pgfpathlineto{\pgfqpoint{2.8pt}{3.2pt}}
    \pgfusepath{stroke}}
\usetikzlibrary{positioning,cd}

\definecolor{UCnavyy}{rgb}{0.094117, 0.168627, 0.2862745}           
\definecolor{navy}{rgb}{0, 0.415686, 0.588235}  

\definecolor{cf9f9f9}{RGB}{249,249,249}
\definecolor{cb3b3b3}{RGB}{179,179,179}
\definecolor{c808080}{RGB}{128,128,128}
\definecolor{c1a1a1a}{RGB}{26,26,26}
\definecolor{cffffff}{RGB}{255,255,255}

\usepackage{flafter}
\usepackage{comment}
\usepackage{cleveref}

\usepackage{natbib}
\bibpunct{(}{)}{;}{a}{,}{,}

\usepackage{thmtools,thm-restate}
\newtheorem{theorem}{Theorem}

\newtheorem{lemma}[theorem]{Lemma}

\newtheorem{example}[theorem]{Example}
\newtheorem{proposition}[theorem]{Proposition}

\newtheorem{assumption}{Assumption}

\newtheorem{definition}[theorem]{Definition}
\newtheorem{remark}[theorem]{Remark}

\newtheorem*{informaltheorem*}{Theorem}

\usepackage{apptools}
\AtAppendix{\counterwithin{theorem}{section}}
\AtAppendix{\counterwithin{conjecture}{section}}
\AtAppendix{\counterwithin{exercise}{section}}
\AtAppendix{\counterwithin{lemma}{section}}
\AtAppendix{\counterwithin{fact}{section}}
\AtAppendix{\counterwithin{claim}{section}}
\AtAppendix{\counterwithin{example}{section}}
\AtAppendix{\counterwithin{proposition}{section}}
\AtAppendix{\counterwithin{corollary}{section}}
\AtAppendix{\counterwithin{problem}{section}}
\AtAppendix{\counterwithin{assumption}{section}}

\AtAppendix{\counterwithin{definition}{section}}
\AtAppendix{\counterwithin{remark}{section}}
\AtAppendix{\counterwithin{observation}{section}}

\AtAppendix{\counterwithin{exercise}{section}}

\usepackage[named]{algo}

\renewcommand{\epsilon}{\varepsilon}
\DeclareMathOperator*{\E}{\mathbb{E}}

\DeclareMathOperator*{\argmin}{\mathrm{arg\,min}}

\DeclareMathOperator*{\arginf}{\mathrm{arg\,inf}}

\newcommand{\ind}{\mathbbm{1}}

\usepackage{mdframed}
\newmdenv[
  topline=false,
  bottomline=false,
  rightline=false,
  skipabove=\topsep,
  skipbelow=\topsep,
  innertopmargin=0pt,
  innerbottommargin=0pt
]{siderules}

\usepackage{ragged2e}

\title{Online nearest neighbor classification}

\author{%
  Sanjoy Dasgupta \hspace{5em} Geelon So\\
  University of California, San Diego\\
  La Jolla, CA 92093\\
  \texttt{\{dasgupta,agso\}@ucsd.edu} 
}

\begin{document}
\maketitle

\begin{abstract}%
  We study an instance of online non-parametric classification in the realizable setting. 
  In particular, we consider the classical 1-\ref{alg:online-nnc} algorithm, and show that it achieves sublinear regret---that is, a vanishing mistake rate---against dominated or smoothed adversaries in the realizable setting.
\end{abstract}

\section{Introduction}
In \emph{online classification}, a learner observes a stream of data points $x_t$ from an instance space $\mathcal{X}$, and it is tasked with sequentially making predictions $\hat{y}_t$ about their classes $y_t$ coming from some label space $\mathcal{Y}$. At each point in time $t = 1,2,\ldots$
\begin{itemize}
    \item[-] the learner is presented with an instance $x_t \in \mathcal{X}$
    \item[-] the learner makes a prediction $\hat{y}_t \in \mathcal{Y}$
    \item[-] the label $y_t$ is revealed, and the learner incurs some loss $\ell(x_t, y_t, \hat{y}_t)$,
\end{itemize}
where $\ell(x, y,\hat{y})$ is a non-negative, bounded loss function satisfying $\ell(x,y,y) = 0$ (there is no penalty for a correct prediction). The learner's performance is given by its \emph{regret} at any time $T$, defined as the difference between the learner's cumulative loss and that of the best fixed classifier $h : \mathcal{X} \to \mathcal{Y}$ that the learner would have chosen  in hindsight from some comparator class $\mathcal{H}$,
\[\mathrm{regret}_T := \sum_{t=1}^T \ell\big(x_t, y_t, \hat{y}_t\big) - \inf_{h \in \mathcal{H}}\, \sum_{t=1}^T \ell\big(x_t, y_t, h(x_t)\big).\]
Learning in the online setting means achieving sublinear regret, $\mathrm{regret}_T = o(T)$, for then the average loss of the online learner is asymptotically no worse than the average loss of the offline learner who had access to the data $(x_1, y_1),\ldots, (x_T, y_T)$ all at once. 

While in the worst-case setting, this sequence of instances and labels may be completely arbitrary, we consider the more restrictive \emph{realizable setting}, in which a concept $c : \mathcal{X} \to \mathcal{Y}$ is fixed at the onset (though it may be chosen adversarially) and describes the labels $y_t = c(x_t)$ for all time.

\begin{figure*}[t]
\begin{center}
\begin{minipage}{0.9\textwidth}
\hrulefill
\begin{algorithm}{nearest neighbor}{
    \label{alg:online-nnc}
    }
    \qfor $t = 1,2,\ldots$\\ 
    \qdo
        receive data point\ \ $x_t$ \\
        compute nearest neighbor \ \
        \smash{$\displaystyle \mathrm{NN}_t := \argmin_{\tau = 1,\ldots, t-1} \, \rho(x_t, x_\tau)$}\\
        make prediction \ \ $\hat{y}_t = y_{\mathrm{NN}_t}$\\
        receive label\ \  $y_t$
    \qrof
\end{algorithm}
\hrulefill
\end{minipage}
\vspace{5pt}

\begin{minipage}{0.85\textwidth}
\small{\textsc{the nearest neighbor rule}. Assume ties in step 3 are broken arbitrarily.} 
\end{minipage}
\end{center}
\end{figure*}

In this paper, we further let $(\mathcal{X}, \rho)$ be a metric space, and we consider online classification through the 1-\ref{alg:online-nnc} rule. This algorithm, first introduced by \cite{fix1951discriminatory}, is a particularly appealing learning algorithm due to its simplicity: this learner memorizes everything it sees. Then, given some instance $x$, it searches for the nearest neighbor among previously seen instances $x_1,\ldots, x_{t}$, returning the corresponding label as the prediction $\hat{y}$. We ask:
\vspace{7pt}
\begin{center}
\begin{siderules}
\begin{minipage}{0.99\textwidth} \raggedright
    \paragraph{Question} What are general conditions under which the 1-\ref{alg:online-nnc} rule achieves sublinear regret in the realizable \ref{protocol:online-learning} setting?
\end{minipage}
\end{siderules}
\end{center}
In our setting, when $\mathcal{H}$ is the family of all nearest-neighbor classifiers, the best hindsight classifier in $\mathcal{H}$ makes no mistakes, and so the regret consists only of the cumulative loss term; we simply aim to understand when the average loss of the nearest neighbor rule converges to zero:
\begin{equation}  \label{eqn:learning}
\mathrm{average\ loss}_T := \frac{1}{T} \sum_{t=1}^T \ell(x_t, y_t, \hat{y}_t) \to 0.
\end{equation}

\subsection{A negative result: the worst-case adversary}
When the comparator class $\mathcal{H}$ can interpolate the sequence of data, learning in the worst-case setting is generally intractable---even in the realizable setting. Unless the learner exactly recovers the underlying concept, a worst-case adversary (or indeed, a best-case teacher) can at each time step find test instances on which the learner errs; the average loss fails to converge to zero. 

\begin{example}[Failing to learn the sign function]
Consider the sign function $\mathrm{sign}(x) := \ind_{x \geq 0}$ on $\mathcal{X} = \mathbb{R}$. The \ref{alg:online-nnc} makes a mistake every round on the sequence of instances:
\[x_t = \big(-1/3\big)^t.\]
At time $t + 1$, the nearest neighbor for $x_{t+1}$ is $x_t$, which has the opposite sign (see \Cref{fig:threshold}).
\end{example}

The above negative example relies on the worst-case adversary's ability to select instances with arbitrary precision in order to construct a hard sequence. For \ref{alg:online-nnc}, the hardness of a point can be related to its separation from points of different classes---constructing a hard sequence like the one above is possible precisely whenever the classes are not separated:

\begin{proposition}[Non-convergence in the worst-case] \label{prop:nonconvergence}
Let $(\mathcal{X}, \rho)$ be a totally bounded metric space and $c$ be a concept. Let $\ell$ be the zero-one loss $\ell(x, y, \hat{y}) = \ind\{y \ne \hat{y}\}$. There is a sequence of instances $(x_t)_t$ on which the \ref{alg:online-nnc} rule fails to achieve sublinear regret on $c$ if and only if there is no positive separation between classes:
\[\inf_{c(x) \ne c(x')}\, \rho(x,x') = 0.\]
\end{proposition}

This makes sense, since the inductive bias built into the nearest neighbor rule is that most points are surrounded by other points of the same class (though one might have to zoom in very close to a point before the labels of its surrounding neighbors become pure). Boundary points are not amenable to the nearest neighbor rule since their labels can't be learned from neighbors, nor do their labels consistently generalize to nearby points. 

Intuitively, the nearest neighbor learner fares poorly if faced with an adversary that can take advantage of boundary points by selecting instances with arbitrary precision. However, it may be able to perform well if its adversary doesn't have unbounded power to find these hard points near the boundary. In this paper, we make this intuition precise through the smoothed analysis of nearest neighbors.

\subsection{Smoothed analysis of online learning}
While the nearest neighbor algorithm does not perform well in all worlds, we might reasonably expect to not live in the worst-case world. In that case, the worst-case analysis of nearest neighbor does not necessarily help elucidate the behavior of the algorithm in practice. 

This motivates the \emph{smoothed analysis} of online learning algorithms, in which the adversary does not directly select instances, but rather distributions $\mu_t$ from which the instances $x_t$ are then drawn. If the distributions are fixed for all time, we recover the i.i.d.\@ setting. If they may be point masses, we recover the worst-case setting. But somewhere in between, the smoothed online setting might also capture more tractable and realistic learning settings, and has been previously studied by \cite{rakhlin2011online,haghtalab2020smoothed,haghtalab2022smoothed,block2022smoothed}.

The following interaction protocol formalizes the \ref{protocol:online-learning} setting:
\begin{figure*}[h]
\begin{center}
\begin{minipage}{0.9\textwidth}
\hrulefill
\begin{protocol}{smoothed online classification}{
    \label{protocol:online-learning}}
    learner selects a prediction strategy $\mathcal{A}$\\
    adversary selects ground truth concept $c$ with knowledge of $\mathcal{A}$\\
    \qfor $t = 1,2,\ldots$\\ 
    \qdo
        adversary selects data distribution $\mu_t$ on $\mathcal{X}$ and draws sample $x_t \sim \mu_t$ \\
        learner makes prediction $\hat{y}_t$ given $x_t$ according to $\mathcal{A}$\\
        learner incurs loss $\ell(x_t, y_t, \hat{y}_t)$ where $y_t = c(x_t)$
    \qrof
\end{protocol}
\hrulefill
\end{minipage}
\vspace{5pt}

\begin{minipage}{0.85\textwidth}
\small{\textsc{smoothed online classification}. By the end of each round, both the adversary and learner sees all of the triple $(x_t, y_t, \hat{y}_t)$. The distribution $\mu_t$ remains hidden to the learner. } 
\end{minipage}
\end{center}
\end{figure*}

One common smoothed setting is the Gaussian perturbation model \citep{spielman2009smoothed}, where the adversary selects $\mu_t$ in the form of a Gaussian $\mathcal{N}(\tilde{x}_t, \sigma^2 I)$. Another natural setting is the $\sigma$-smoothed adversary model \citep{haghtalab2020smoothed}, where there is some base distribution $\nu$ on the instance space $\mathcal{X}$, and the adversary is constrained to not boost the probability mass of any region $A \subset \mathcal{X}$ by more than a multiplicative factor $\sigma^{-1}$, so $\mu_t(A) \leq \sigma^{-1} \cdot \nu(A)$. 

We distill the key property of these smoothed adversaries through the notion of a \emph{dominated adversary} on a measure space $(\mathcal{X}, \nu)$. The dominated adversary is simply one that cannot place a constant probability mass $\mu(A)$ on region $A \subset \mathcal{X}$ with arbitrarily small $\nu$-mass. We define:

\begin{definition}[Dominated adversary] \label{def:dominated-adversary} Let $(\mathcal{X}, \nu)$ be a measure space. The measure $\nu$ \emph{uniformly dominates} a family $\mathcal{M}$ of probability distributions on $\mathcal{X}$ if for all $\epsilon > 0$ there exists $\delta > 0$ such that:
\[\nu(A) < \delta \quad \Longrightarrow \quad \mu(A) < \epsilon,\]
for all $A \subset \mathcal{X}$ measurable and distribution $\mu \in \mathcal{M}$. A \ref{protocol:online-learning} adversary is \emph{$\nu$-dominated} if at all times $t$ it selects $\mu_t$ from a family of distributions uniformly dominated by $\nu$.
\end{definition}

To see why this is helpful, let's say that $A_t \subset \mathcal{X}$ is the set of points on which the learner makes mistakes at time $t$. For a learner's error rate to converge to zero against a dominated adversary, it suffices to prove that the sequence $\nu(A_t)$ converges to zero: the probability that the dominated adversary induces a mistake $\mu_t(A_t)$ must also converge to zero since the $\mu_t$'s are uniformly dominated by $\nu$. Convergence of the average loss then follows from the law of large numbers for martingales.

Of course, this captures only a narrow set of scenarios where learning succeeds---in general, the convergence of the mistake region to a null set is a much stronger than the convergence of the mistake rate to zero. For example, if the adversary never tests on some region of the space, the average loss could still converge to zero even though the size of the mistake region might not. Instead, we shall argue that under mild boundary conditions, all but finitely many mistakes that a nearest neighbor learner makes must come from a very small set of `hard points' (small with respect to $\nu$). But as the adversary is $\nu$-dominated, those instances can come only very infrequently.

\begin{figure}[t]
    \centering
    \begin{tikzpicture}
    \draw (-7.2,0) -- (7.2,0);
    \filldraw [gray] (0,0) circle (2pt);
    \node at (0, -0.5) {0};
    
    \filldraw [black] (-7,0) circle (2pt);
    \node at (-7, 0.5) (x1) {$x_1$};

    \filldraw [black] (7/2.5,0) circle (2pt);
    \node at (7/2.5, 0.5) (x2) {$x_2$};

    \filldraw [black] (-7/6.25,0) circle (2pt);
    \node at (-7/6.25, 0.5) (x3) {$x_3$};

    \filldraw [black] (7/15.625, 0) circle (2pt);
    \node at (7/15.625, 0.5) (x4) {$\dotsm$};

    \path[->] (x1) edge[bend left] node [left] {} (x2);
    \path[->] (x2) edge[bend right] node [left] {} (x3);
    \path[->] (x3) edge[bend left] node [left] {} (x4);
    
    \end{tikzpicture}
    \caption{Learning the sign function $\ind_{\{x \geq 0\}}$ on $\mathbb{R}$. The nearest neighbor classifier makes a mistake every single round on the sequence $x_t = (-1/3)^t$, where each subsequent test point alternates sign.}
    \label{fig:threshold}
\end{figure}
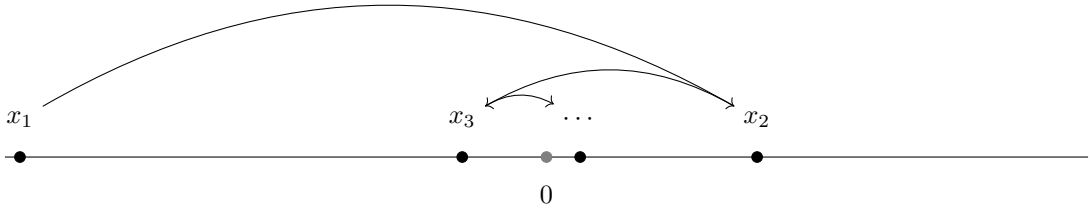

\subsection{Main results}
Let $(\mathcal{X}, \rho, \nu)$ be a metric measure space. Assume that $\rho$ is a separable metric and $\nu$ is a finite Borel measure. We prove that under mild boundary conditions, the \ref{alg:online-nnc} rule achieves sublinear regret in the \ref{protocol:online-learning} setting against $\nu$-dominated adversaries. 

To state the boundary condition, let's formalize the notion of boundary points. Given a concept $c : \mathcal{X} \to \mathcal{Y}$, define the margin $m_c(x)$ of a point $x \in \mathcal{X}$ as its distance to points of different classes:
\[m_c(x) := \inf_{c(x) \ne c(x')}\, \rho(x, x').\]
We say that $x$ is a \emph{boundary point} of $c$ if $m_c(x) = 0$, which is to say that it is arbitrarily close to points of other classes. Denote the set of boundary points by $\partial \mathcal{X}$. The condition we require is this:
\begin{assumption}[Boundary condition] \label{ass:boundary}
The set of boundary points $\partial \mathcal{X}$ is essentially countable. That is, it is the union of a countable set and a $\nu$-measure zero set.
\end{assumption}

This boundary condition is the same condition required by \cite{cover1967nearest} to prove the consistency of 1-\ref{alg:online-nnc} in the i.i.d.\@ setting. We can now state our main result:

\begin{theorem}[Convergence of nearest neighbor] \label{thm:conv-nn}
Let $(\mathcal{X}, \rho, \nu)$ be a metric measure space, where $\rho$ is a separable metric and $\nu$ is a finite Borel measure. Let $c : \mathcal{X} \to \mathcal{Y}$ satisfy Assumption~\ref{ass:boundary}. Then, the \ref{alg:online-nnc} rule achieves sublinear regret when learning $c$ against a $\nu$-dominated adversary. In particular, the average loss converges to zero:
\[\phantom{\quad\mathrm{a.s.}}\lim_{T\to\infty} \, \frac{1}{T} \sum_{t=1}^T \ell(x_t, y_t, \hat{y}_t) = 0\quad \mathrm{a.s.}\]
\end{theorem}

To show this, we prove a general condition in \Cref{sec:olc-conv} under which online learning is possible against a dominated adversary.  \Cref{sec:nnc-olc} shows that \ref{alg:online-nnc} satisfies this condition. We also derive rates of convergence in \Cref{sec:rates}. Here is a simple instantiation of more general rates:

\begin{theorem}[Rate of convergence for nearest neighbor]
     Let $\mathcal{X} \subset \mathbb{R}^d$ be the unit ball. Assume $d > 1$. Let the set of boundary points of $c : \mathcal{X} \to \mathcal{Y}$ have finite Minkowski content with respect to the Lesbegue measure and let the adversary be $\sigma$-smoothed. Let $p > 0$. With probability at least $1 - p$, the \ref{alg:online-nnc} rule satisfies the error rate bound simultaneously for all time $T$:
    \[\frac{1}{T} \sum_{t=1}^T \ell(x_t,y_t, \hat{y}_t) \leq \left(\frac{T}{\sigma}\right)^{(-1 + o(1))/(d+1)}.\]
\end{theorem}

\subsection{Related works}
The 1-nearest neighbor rule \citep{fix1951discriminatory} was shown by \cite{cover1967nearest} to be consistent when the instances come i.i.d. under \Cref{ass:boundary}. On the other hand, in the online learning setting where the sequence of instances can be arbitrary \citep{littlestone1988learning,cesa2006prediction}, there is no learning algorithm that can achieve sublinear regret in the worst-case even in the case of learning a threshold function. However, worst-case analyses of algorithms can fail to explain the observed behavior of algorithms, especially if hard instances are extremely rare in practice \citep{spielman2009smoothed, roughgarden2021beyond}. This motivates the smoothed analysis of algorithms, first introduced by \cite{spielman2004smoothed}. The setting of smoothed online learning was first studied by \cite{rakhlin2011online}, and has recently been followed up by a series of work \citep[and references therein]{haghtalab2020smoothed,haghtalab2022smoothed,block2022smoothed}. Our work fills in the gap between the i.i.d.\@ and worst-case analysis of nearest neighbor, while also giving the first convergence result in smoothed non-parametric online learning. \Cref{app:related-works} further expands on related works.

\section{Preliminaries} \label{sec:prelims}
In \ref{protocol:online-learning}, the learner incrementally updates its prediction rule as it receives more data. It does so according to a prediction strategy $\mathcal{A}$, which constructs each subsequent hypothesis $h_{t+1}: \mathcal{X} \to \mathcal{Y}$ based on previously seen data:
\begin{equation*} 
\mathcal{A} : \big\{(x_\tau, y_\tau)\big\}_{\tau=1}^t \mapsto h_{t+1}.
\end{equation*}
Suppose that $c$ is the underlying concept to be learned. Then, every hypothesis $h : \mathcal{X} \to \mathcal{Y}$ induces an error function $\mathcal{E}: \mathcal{X} \to \mathbb{R}$, which is the loss that $h$ achieves at any particular instance $x$,
\[\mathcal{E}(x) := \ell\big(x, c(x), h(x)\big)\]
When the prediction strategy $\mathcal{A}$ and concept $c$ are clear from context, it shall be fruitful to let $\mathcal{E}_t$ be the associated error function to $h_t$ generated by $\mathcal{A}$. Rewriting \Cref{eqn:learning}, we say that the strategy $\mathcal{A}$ learns if it achieves a vanishing error rate:
\[\mathrm{average\ loss}_T = \frac{1}{T} \sum_{t=1}^T \mathcal{E}_t(x_t) \to 0.\]

\subsection{Online local consistency}
We introduce the \emph{online local consistency} (OLC) condition for learning against dominated adversaries. This is a condition that depends on both the learning algorithm and the concept to be learned. 

For intuition, let $\mathcal{X}$ be composed of (countably many) known clusters, and suppose that we are guaranteed that points in the same cluster have the same label. A natural learning algorithm is to remember a single label from each cluster, and to return that label if a point from the same cluster is queried. In this setting, the learner makes at most one mistake per cluster. If $\nu$ is a finite measure over $\mathcal{X}$, then over time, a $\nu$-dominated adversary will find it increasingly harder to pick points from previously unseen clusters; the mistake rate will eventually converge to zero.

We generalize these easily-learned clusters through the notion of \emph{locally-learned sets} for a learner. In the following, if $U \subset \mathcal{X}$ is a locally-learned set, we can think of the online learning problem restricted to $U$ as easy for the learner: no matter what sequence of points an adversary chooses, the learner will eventually incur arbitrarily small loss from $U$.

\begin{definition}[Locally-learned set]
Let $c : \mathcal{X} \to \mathcal{Y}$ be a concept. We say that $c$ is \emph{locally learned} on a subset $U \subset \mathcal{X}$ by the prediction strategy $\mathcal{A}$ when, for any sequence of instances $(x_t)_t$, either:
\begin{enumerate}
    \item[(i)] $x_t$ falls into $U$ finitely often, or
    \item[(ii)] the error function $\mathcal{E}_t\big|_U \to 0$ restricted to $U$ uniformly converges to zero.
\end{enumerate}
In this case, we say that $U$ is a \emph{locally-learned set} for $c$.
\end{definition}

For example, singleton sets are locally-learned by \emph{consistent learners}, which are learners that exactly interpolate past data. But in general, if $\mathcal{X}$ is uncountable, this family of locally-learned sets is too granular to work with, as the family also becomes uncountably large. The OLC condition ensures that there is a way to cut up the problem into a countable collection of `easy' problems.

\begin{definition}[Online local consistency]
A prediction strategy $\mathcal{A}$ is \emph{online locally consistent (OLC)} for a concept $c$ if there exists a countable collection $\mathcal{U}_c := \{U_n\}_n$ of locally learned sets for $c$ that covers all but a $\nu$-negligible subset of $\mathcal{X}$. 
\end{definition}

The argument for why an OLC learner can perform well against a dominated adversary is not unlike the earlier example of learning labels for pure clusters. We can restrict the learning problem to a finite collection of locally-learned sets that covers all but a small part of $\mathcal{X}$. Because the part of $\mathcal{X}$ we covered consists only of finitely many easy learning problems, the learner's error rate will eventually converge to zero here. The uncovered portion of $\mathcal{X}$ can be made sufficiently small so that its contribution to the error rate is made arbitrarily small---the adversary cannot test the learner with instances from this region very frequently because it is $\nu$-dominated.

\subsection{Mutually-labeling sets}

For the analysis of nearest neighbor, we introduce the notion of a \emph{mutually-labeling set}. It is a set defined so that, upon receiving a label for any point within the set, the nearest neighbor learner will never make a subsequent mistake on any other point in that set (see \Cref{fig:mutually-labeling}).
\begin{definition}[Mutually-labeling set]
A set $U \subset \mathcal{X}$ is a \emph{mutually-labeling set} for a concept $c$ if:
\[\phantom{,\qquad \forall x, x' \in U.}\rho(x,x') < m_c(x),\qquad \forall x, x' \in U.\]
\end{definition}

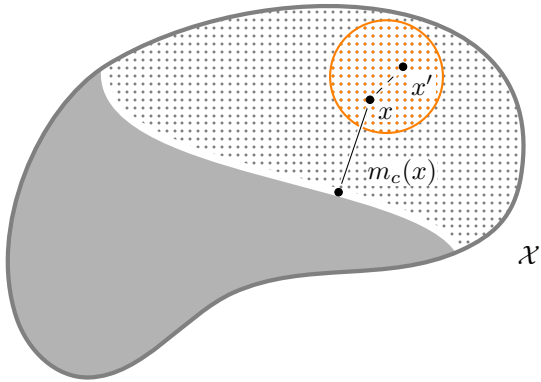
\begin{figure}[ht]
    \centering
    \begin{tikzpicture}[y=0.50pt, x=0.50pt,yscale=-2.5, xscale=2.5, inner sep=0pt, outer sep=0pt, trim left = 2cm, 
point/.style = {circle, draw=black, inner sep=1, outer sep=2,
                fill=black,
                node contents={}}]

  \path[fill=cb3b3b3] 
  (54.9174,61.7760) .. controls (34.8042,74.5615) and (18.9199,110.7493) .. (28.0229,137.9947) .. controls (31.3683,148.0074) and (39.9622,156.3286) .. (49.8747,156.1061) .. controls (66.2536,155.7384) and (79.0581,136.9671) .. (93.5782,130.4483) .. controls (115.6379,120.5446) and (140.3069,127.6766) .. 
  (162.4953,117.6194) .. controls (150,95) and  (50,95) ..
  (54.9174,61.7760) -- cycle;

  \path[pattern={dots}, pattern color=gray]   
  (162.4953,117.6194) .. controls (172.1698,113.2343) and (178.6216,107.5658) .. (180.6491,95.3575) .. controls (190.9266,33.4721) and (99.8417,33.6483) .. (54.9174,61.7760) .. controls (50,95) and (150,95) .. (162.4953,118.6194) -- cycle;

  \path[draw=white, line cap=butt,line join=miter,line width=3pt] 
  (54.9174,61.7760) .. controls (50,95) and (150,95) .. (162.4953,118.6194);

  \path[draw=gray,ultra thick]
  (54.9174,61.7760) .. controls (34.8042,74.5615) and (18.9199,110.7493) .. (28.0229,137.9947) .. controls (31.3683,148.0074) and (39.9622,156.3286) .. (49.8747,156.1061) .. controls (66.2536,155.7384) and (79.0581,136.9671) .. (93.5782,130.4483) .. controls (115.6379,120.5446) and (140.3069,127.6766) .. (162.4953,117.6194) .. controls (172.1698,113.2343) and (178.6216,107.5658) .. (180.6491,95.3575) .. controls (190.9266,33.4721) and (99.8417,33.6483) .. (54.9174,61.7760) -- cycle;

  \node [label=$\mathcal{X}$] at (183, 122) {};

  \node[circle, draw=orange, thick, pattern={dots}, pattern color=orange, minimum size=1.5cm] (c) at (140,65) {};

  \node (x) at (135, 72)  [point=black];
  \node (y) at (145, 62)  [point=black];
  \node at (x) [inner sep=1pt, outer sep=2pt,fill=white,below right] {$x$};
  \node at (y) [inner sep=1pt,outer sep=2pt,fill=white,below right] {$x'$};
  \node (boundary) at (125.5, 100) [point=black];
  \draw[color=white, ultra thick] (x) -- (y);
  \draw[color=white, ultra thick] (x) -- (boundary);
  \draw[color=black, dashed] (x) -- (y);
  \draw[color=black] (x) -- node[outer sep=5pt,fill=white,below right] {$m_c(x)$} (boundary);

\end{tikzpicture}
    \caption{The instance space $\mathcal{X}$ is divided into two classes, the solid region in the lower left and the dotted region in the upper right. The orange ball is an example of a mutually-labeling set. Suppose \ref{alg:online-nnc} previously received the label for $x'$. Then, it shall always classify $x$ correctly in the future; $x$ can never have a nearer neighbor of a different class.}
    \label{fig:mutually-labeling}
\end{figure}

Naturally, mutually-labeling sets are locally learned (\Cref{lem:mlp-ll}). The proof of convergence for OLC learners using locally-learned sets generalizes the following proof sketch for \ref{alg:online-nnc}:

\paragraph{Proof Sketch of \Cref{thm:conv-nn}} For simplicity, let's assume a stronger boundary condition: the set of boundary points of $c$ has $\nu$-measure zero. It turns out that if $x$ is not a boundary point, then sufficiently small open balls centered at $x$ are mutually-labeling sets (see Lemma~\ref{lem:ball-mutually-labeling}). Thus, $\mathcal{X}$ is covered almost everywhere by open mutually-labeling sets. By separability of $\rho$ and finiteness of $\nu$, all but an arbitrarily small region of $\mathcal{X}$ can be covered by a finite number of such sets.

Because the \ref{alg:online-nnc} learner makes at most one mistake on each mutually-labeling set, eventually all mistakes must come from the uncovered hard region. The average rate at which a $\nu$-dominated adversary can test the learner with these hard instances can almost surely be bounded above by any $\epsilon > 0$, by selecting a sufficiently small hard region for our analysis. Thus, the average loss converges to zero almost surely, by the law of large numbers for martingales. \hfill $\blacksquare$

\section{Convergence of OLC learners} \label{sec:olc-conv}
\begin{theorem}[Convergence of error rate] \label{thm:error-rate-convergence}
Given an \ref{protocol:online-learning} problem on the measure space $(\mathcal{X}, \nu)$ where $\nu$ is a finite measure. Suppose the learner is online locally consistent with respect to $c$ and that the adversary is $\nu$-dominated. Then, the learner's error rate converges:
\[\phantom{\quad\mathrm{a.s.}}\lim_{T\to\infty} \, \frac{1}{T} \sum_{t=1}^T \ell(x_t, y_t, \hat{y}_t) = 0\quad \mathrm{a.s.}\]
\end{theorem}

Before commencing the proof, recall that $\mathcal{E}_t(x_t)$ is the error incurred by the learner at time $t$. Given error function $\mathcal{E}$ and test distribution $\mu$, let's also define the notation $\mathcal{E}(\mu)$ to be the expected error,
\[\mathcal{E}(\mu) := \E_{x \sim\mu}[\mathcal{E}(x)].\]
If $A \subset \mathcal{X}$ is measurable, let $\mathcal{E}\ind_A$ denote the pointwise product of $\mathcal{E}$ and the indicator on $A$.

\paragraph{Proof of \Cref{thm:error-rate-convergence}}
We show that for any $\epsilon > 0$, the following error rate bound holds:
\begin{equation} \label{eqn:error-convergence}
\phantom{\textrm{a.s.} \quad}\lim_{T \to \infty}\, \frac{1}{T} \sum_{t=1}^T \mathcal{E}_t(x_t) < 2\epsilon \quad \textrm{a.s.}
\end{equation}
If so, then this statement holds simultaneously for any countable sequence of $\epsilon$ converging to zero, implying that the error rate converges to zero almost surely. 

To prove \Cref{eqn:error-convergence}, fix $\epsilon > 0$. Because the loss function is bounded above, say by $C > 0$, we have for any error function $\mathcal{E}$ and any measurable $A \subset \mathcal{X}$,
\[(\mathcal{E} \ind_A)(\mu) \leq (C \ind_A)(\mu) = C \cdot \mu(A).\]
The right-hand side can be bounded in terms of $\nu(A)$ whenever $\mu$ is chosen by a $\nu$-dominated adversary. In particular, we may select $\delta > 0$ such that:
\begin{equation} \label{eqn:risk-bound}
\nu(A) < \delta \quad \Longrightarrow \quad (\mathcal{E}\ind_A)(\mu) < \epsilon.
\end{equation}
Let us do so: any region whose $\nu$-mass is less than $\delta$ contributes no more than $\epsilon$ to the error rate.

We claim that there exists a subset $V \subset \mathcal{X}$ with the properties that (a) there exists a random time $T_\epsilon$ such that the learner incurs less than $\epsilon$ error for any further instance $x_t$ that lands in $V$,
\[\phantom{,\qquad \forall t \geq T_\epsilon \textrm{ and } x_t \in V.}\mathcal{E}_t(x_t) < \epsilon,\qquad \forall t > T_\epsilon \textrm{ and } x_t \in V,\]
and that (b) $V$ covers all but a $\delta$-mass of $\mathcal{X}$, so that $\nu(V^c) < \delta$. Assume this for now---we decompose $\mathcal{E}_t$ into its pieces on $V$ and $V^c$, with $\mathcal{E}_t = \mathcal{E}_t \ind_{V} + \mathcal{E}_t \ind_{V^c}$. We have:
\begin{itemize}
    \item[-] By property (a) of $V$, the sequence $(\mathcal{E}_t \ind_V) (x_t)$ eventually remains less than $\epsilon$, in particular when we have $t > T_\epsilon$. Because $T_\epsilon$ is almost surely finite, we have that:
    \begin{equation} \label{eqn:V-inequality}
    \lim_{T\to\infty} \, \frac{1}{T}\sum_{t=1}^T (\mathcal{E}_t\ind_V)(x_x) < \epsilon.
    \end{equation}
    \item[-] By property (b) of $V$, the mass of $V^c$ is less than $\delta$. \Cref{eqn:risk-bound} implies:
    \[(\mathcal{E}_t \ind_{V^c})(\mu_t)< \epsilon.\]
    By the law of large numbers for martingales (\Cref{thm:slln}), this implies that almost surely:
    \begin{equation} \label{eqn:Vc-inequality}
    \lim_{T\to\infty} \, \frac{1}{T} \sum_{t=1}^T (\mathcal{E}_t \ind_{V^c})(x_t) = \lim_{T\to\infty} \, \frac{1}{T} \sum_{t=1}^T (\mathcal{E}_t \ind_{V^c})(\mu_t) < \epsilon.
    \end{equation}
\end{itemize}
Because the loss function is bounded, the error rates within the limits in \Cref{eqn:V-inequality,eqn:Vc-inequality} are also bounded. Thus, we can sum the two equations and apply dominated convergence, interchanging limits and sum, to yield \Cref{eqn:error-convergence}.

To finish the proof, we show that $V$ exists. The learner is OLC, so there is a countable cover $\{U_n\}_{n \in \mathbb{N}}$ of locally learned sets for $\mathcal{X}$ almost everywhere. Let $V$ satisfying $\nu(V^c) < \delta$ be chosen as a finite union:
\[V := \bigcup_{i=1}^N U_n.\]
Such a $N < \infty$ exists by the continuity of measure, since $\bigcup_{n=1}^\infty U_n$ is essentially all of $\mathcal{X}$. 

By now, we have constructed $V$ in such a way such that property (b) holds. To show property (a), we use the fact that each $U_n$ is locally learned: either (i) $(x_t)_t$ eventually never returns to $U_n$, which is to say that $\ind\{x_t \in U_n\}$ converges to zero over time, or (ii) for sufficiently large $t$, $\mathcal{E}_t\big|_{U_n} < \epsilon$. Thus, almost surely, there exists some $T_{n}$ such that for all $t > T_{n}$,
\[\mathcal{E}_t(x_t) \cdot \ind\{x_t \in U_n\} < \epsilon.\]
Property (a) follows by defining $T_\epsilon := \max\{T_1,\ldots, T_N\}$.
\hfill $\blacksquare$

\section{Nearest neighbor is an OLC learner} \label{sec:nnc-olc}
\begin{theorem}[Nearest neighbors is OLC] \label{thm:nn-olc}
Let $(\mathcal{X}, \rho, \nu)$ be a metric measure space, where $\rho$ is a separable metric and $\nu$ is a finite Borel measure. If $c$ is a concept whose boundary points satisfy Assumption~\ref{ass:boundary}, then \ref{alg:online-nnc} is OLC with respect to $c$.
\end{theorem}

To show that \ref{alg:online-nnc} is OLC, we need to prove that any concept $c$ with essentially countable boundary also has a countable family of locally-learned sets. 

We define two types of locally-learned sets for \ref{alg:online-nnc}: singleton sets for the boundary points and mutually-labeling sets for everything else. Recall that mutually-labeling sets $U$ satisfy:
\[\phantom{,\qquad \forall x,x' \in U.}\rho(x,x') < m_c(x),\qquad \forall x,x' \in U,\]
where $m_c(x)$ is the margin between $x$ and the boundary of $c$. Note that all points in $U$ share the same label. If this weren't the case, then there would exist $x, x' \in U$ with different labels such that:
\[\rho(x,x') < \underbrace{\inf_{c(x) \ne c(\tilde{x})}\, \rho(x, \tilde{x})}_{m_c(x)} \leq \rho(x,x'),\]
a contradiction. The following lemma further shows that these are locally learned sets:

\begin{lemma}[Mutually labeling property] \label{lem:mlp-ll}
Consider learning the concept $c$ via the \ref{alg:online-nnc} rule. If $U$ is a mutually-labeling set for $c$ and $x_t \in U$, then for all time $\tau > t$, the predictor $h_\tau$ is correct on all of $U$. Thus, $U$ is locally learned.
\end{lemma}
\begin{proof}
Let $x \in U$ so that $c(x) = c(x_t)$. When $\tau > t$, the nearest neighbor classifier errs on $x$ only if the closest point to $x$ among $x_1,\ldots, x_\tau$ is of the opposite class. But this is impossible since the closest point must be no more than a distance of $\rho(x,x_t)$ and $U$ is mutually labeling.
\end{proof}

Sufficiently small balls around any non-boundary point $x$ are mutually-labeling sets.
\begin{lemma}[Mutually labeling balls] \label{lem:ball-mutually-labeling}
Let $c : \mathcal{X} \to \mathcal{Y}$ be a concept, and suppose that $x$ has positive margin $m_c(x) > 0$. Then, the open ball $B\big(x, m_c(x)/3\big)$ is mutually labeling.
\end{lemma}

\begin{proof}
Let $x_1, x_2 \in B\big(x, m_c(x)/3\big)$. By the triangle inequality,
\[\rho(x_1, x_2) \leq \rho(x_1, x) + \rho(x, x_2) < 2m_c(x) /3.\]
We also know for $i \in \{1,2\}$ and for all $\tilde{x}$ that $\rho(x_i, \tilde{x}) \geq \rho(x, \tilde{x}) - \rho(x_i, x)$, by the reverse triangle inequality. Since $c(x_i) = c(x)$, we take infimums on both sides over $\tilde{x}$ where $c(\tilde{x}) \ne c(x)$, so:
\[\underbrace{\inf_{c(x_i) \ne c(\tilde{x})}\, \rho(x_i, \tilde{x})}_{m_c(x_i)} \geq \underbrace{\inf_{c(x) \ne c(\tilde{x})}\, \rho(x,\tilde{x})}_{m_c(x)} \, -\, \rho(x_i, x) \geq 2m_c(x) / 3.\]
This implies that $\rho(x_1, x_2) < m_c(x_1)$, so that $B\big(x, m_c(x)/3\big)$ is mutually labeling.
\end{proof}

\paragraph{Proof of \Cref{thm:nn-olc}}
Given a concept $c$ with essentially countable boundary, we construct a countable cover of $\mathcal{X}$ except for a $\nu$-measure zero set by locally-learned sets of $c$. 

Let us denote by $\partial \mathcal{X}$ the set of boundary points $\{x : m_c(x) = 0\}$. By Lemma~\ref{lem:ball-mutually-labeling}, non-boundary points $\mathcal{X} \setminus \partial \mathcal{X}$ can be covered by the family of open mutually-labeling sets,
\[\big\{B(x, m_c(x)/3) : x \in \mathcal{X} \setminus \partial \mathcal{X}\big\}.\] 
By the separability of $\mathcal{X}$, there is a countable subcover of $\mathcal{X} \setminus \partial \mathcal{X}$ by mutually-labeling sets. These are locally-learned sets, by Lemma~\ref{lem:mlp-ll}.

As for the boundary points, the set $\partial\mathcal{X}$ is essentially countable $\partial \mathcal{X} = \mathcal{N} \cup \mathcal{Z}$, where $\mathcal{N}$ is countable and $\mathcal{Z}$ is $\nu$-measure zero. Then, each $\{x\}$ for $x \in \mathcal{N}$ is a locally-learned set because nearest neighbors is a consistent learner. Together, these two collections of locally-learned sets is a countable cover of all of $\mathcal{X}$ except for a measure zero set; thus, the \ref{alg:online-nnc} is OLC.
\hfill $\blacksquare$

\section{Rates of convergence for nearest neighbor} \label{sec:rates}
Rates of convergence for \ref{alg:online-nnc} arise almost immediately out of the proof technique for asymptotic convergence. Recall that the proof technique consisted of decomposing $\mathcal{X}$ into $V$ and $V^c$, where (i) $V$ can be covered by finitely many mutually-labeling sets and (ii) $V^c$ has small $\nu$-mass. 

The proof can be adapted to yield rates by quantifying (i) the number of mutually-labeling sets required to cover $V$, and (ii) the rate at which a $\nu$-dominated adversary can boost the probability of selecting points from $V^c$. To bound these, we respectively define the following:

\begin{definition}[Mutually-labeling covering number]
    Let $V\subset \mathcal{X}$. The \emph{mutually-labeling covering number} $\mathcal{N}_{\mathrm{ML}}(V)$ given a concept $c$ is the size of a minimal covering of $V$ by mutually-labeling sets.
\end{definition}

\begin{definition}[Smoothness rate]
    An adversary has \emph{smoothness rate} $\epsilon : \mathbb{R}_{\geq 0} \to [0,1]$ whenever all distributions $\mu$ it can select satisfy:
    \[\phantom{\qquad \forall A \subset \mathcal{X} \textrm{ measurable}}\mu(A) \leq \epsilon\big(\nu(A)\big),\qquad \forall A \subset \mathcal{X} \textrm{ measurable}.\]
\end{definition}

An adversary is $\nu$-dominated if $\displaystyle \lim_{\delta \to 0}\, \epsilon(\delta) = 0$. It is $\sigma$-smooth if $\epsilon$ is further $\frac{1}{\sigma}$-Lipschitz.

For simplicity, let us assume that the boundary $\partial \mathcal{X}$ has $\nu$-measure zero. Then, the following mistake rate is obtained by separately counting mistakes on $V$ and $V^c$:
\[\E\big[\# \textrm{mistakes by time } T\big] \leq \min \left\{T\, ,\,  \inf_{V \subset \mathcal{X}} \, \mathcal{N}_{\mathrm{ML}}(V) + T\epsilon\big(\nu(V^c)\big)\right\}.\]
By a standard application of Azuma-Hoeffding's, we can convert this into a high-probability bound:
\begin{theorem}[Convergence rate] \label{thm:nn-conv-rate}
    Let $(\mathcal{X}, \rho, \nu)$ be a metric measure space with separable metric $\rho$ and finite Borel measure $\nu$. Let $c$ be a concept with measure zero boundary. Let the $\nu$-dominated adversary have smoothness rate $\epsilon$. Fix $p > 0$. Then, with probability at least $1 - p$, the following mistake bound holds for \ref{alg:online-nnc} simultaneously for all $T \in \mathbb{N}$ :
    \[\# \mathrm{mistakes}_T \leq \min\left\{T\,,\, \inf_{V \subset \mathcal{X}}\, \mathcal{N}_{\mathrm{ML}}(V) + T \epsilon\big(\nu(V^c)\big) + \sqrt{2T \log \frac{2T}{p}}\right\}.\]
\end{theorem}

\subsection{Convergence rate for length metric spaces}
In this section, we instantiate the convergence rate when $\mathcal{X}$ is a length metric space. The appealing property of length spaces is that the margin of a point $x$ is simply its distance to boundary points:
\begin{lemma}[Margin in length spaces] \label{lem:margin-length-space}
    Let $(\mathcal{X},\rho)$ be a length space. Let $c$ be a classifier. Then,
    \[m_c(x) = \rho(x, \partial \mathcal{X}).\]
\end{lemma}
In this case, it is natural to restrict $V \subset \mathcal{X}$ in \Cref{thm:nn-conv-rate} to the sets of the form:
\[V_r := \big\{x \in \mathcal{X} : m_c(x) \geq r\big\}.\]
These are the set of points whose margin is at least $r$. Then, we need to control the mutual-labeling covering number of $V_r$ and the $\nu$-masses of $V_r^c$. When $\mathcal{X}$ is a length space, these can be bounded in terms of the geometry of the boundary $\partial \mathcal{X}$. The reason is that in length spaces, points with small margins are also close to boundary points: here, $V_r^c$ precisely coincides with the \emph{$r$-expansion} $\partial \mathcal{X}^r$ of the boundary. And when $\mathcal{X}$ is a doubling space, we can quantify the bounds in terms of the \emph{box-counting dimension} $d(\partial \mathcal{X})$ and the \emph{Minkowski content} $\mathfrak{m}(\partial \mathcal{X})$ of the boundary. 

In particular, \Cref{prop:geometric-boundary} shows that for small $r$,
\begin{equation}  \label{eqn:approximate-geometric-bounds}
\mathcal{N}_{\mathrm{ML}}\big(V_r\big) \lesssim r^{-d} \qquad \textrm{ and }\qquad \nu\big(V_r^c\big) \lesssim \mathfrak{m}\cdot r,
\end{equation}
where the hand-waving inequality can be made rigorous by replacing $d = d + o(1)$ and $\mathfrak{m} = \mathfrak{m} + o(1)$. For example, this yields convergence rates of \ref{alg:online-nnc} against $\sigma$-smoothed adversaries, by plugging \Cref{eqn:approximate-geometric-bounds} into \Cref{thm:nn-conv-rate}. After optimizing $r$, we obtain the following result:
\[\#\mathrm{mistakes}_T \lesssim \left(\frac{\mathfrak{m}T}{\sigma}\right)^{d/(d + 1)}.\]

\begin{theorem}[Convergence rate against $\sigma$-smoothed adversaries] \label{thm:conv-nn-sigma}
    Let $(\mathcal{X}, \rho, \nu)$ be a bounded length space with finite doubling dimension and Borel measure. Suppose the concept $c$ satisfies $\nu(\partial \mathcal{X}) = 0$. Let the adversary be $\sigma$-smooth for $\sigma > 0$. Denote the box-counting dimension and Minkowski content of $\partial \mathcal{X}$ by $d := d(\partial \mathcal{X})$ and $\mathfrak{m} := \mathfrak{m}(\partial \mathcal{X})$ respectively. Assume $d > 1$. 
    
    The following holds for \ref{alg:online-nnc}: given $c_1, c_2, p > 0$, there exist constants $C_0, C_1 > 0$ such that with probability at least $1 - p$, the mistake bound holds simultaneously for all $T$:
    \[\#\mathrm{mistakes}_T \leq C_0 + C_1 \left(\frac{(\mathfrak{m} + c_2)T}{\sigma}\right)^{(d + c_1) / (d + 1)}.\]
\end{theorem}

See \Cref{sec:app-metric-measure} for proofs.

\newpage
\bibliography{references}

\newpage
\appendix
\section{Related work} \label{app:related-works}
\paragraph{Non-parametric online learning} We consider non-parametric online classification in the realizable setting with bounded loss. Without further conditions imposed on the problem, existing work shows that online learning as a rule is not possible in this setting. Consider the setting with an unrestricted adversary and the zero-one loss $\ell(x,y, \hat{y}) = \ind\{y \ne \hat{y}\}$. For binary classification in this case, \cite{bousquet2021theory} has characterized online learnability of a concept class $\mathcal{C}$ by the non-existence of infinite Littlestone trees associated to $\mathcal{C}$, a weaker condition than that of having finite Littlestone dimension \citep{littlestone1988learning,ben2009agnostic}. However, as any reasonably non-parametric setting will have infinite Littlestone trees, there is not much more to be said about online non-parametric classification with the worst-case adversary under the zero-one loss.

And so, because of the difficulty of online non-parametric learning, conditions are often imposed that (i) restrict the concept class, (ii) relax the notion of regret, or (iii) constrain the adversary.

In the first instance, the difficulty of making inferences can be reduced by imposing regularity conditions such as Lipschitzness or smoothness on the underlying concept class. This is especially natural in the regression setting where the label space $\mathcal{Y}$ is continuous. For example, \cite{kulkarni1995rates} consider the noisy setting where the label $y$ associated to an instance $x$ is drawn from the conditional distribution $P_{Y|X=x}$ where the conditional mean $\E[Y|X = x]$ is Lipschitz continuous in $x$. In the realizable classification, this constraint guarantees that points of different classes have positive separation.

In the second instance, the notion of what it means to learn online can be relaxed by changing the definition of regret. For example, much of the existing work in online non-parametric learning assume universal Lipschitz or H\"older constants constraining the family of comparator functions $\mathcal{H}$, while also considering a convex or Lipschitz loss function \citep[and references therein]{hazan2007online,vovk2007competing,rakhlin2015online,gaillard2015chaining,kuzborskij2020locally}. We shall make no such assumptions in this work.

In the last instance, the hardness of the online sequence of points is limited, as in the smoothed online setting of \cite{rakhlin2011online, haghtalab2020smoothed, haghtalab2022smoothed, block2022smoothed}. In particular, \cite{haghtalab2022smoothed} show that any concept class is online learnable against smoothed adversaries if it has finite VC dimension \citep{vapnik1971uniform}. But because any reasonably non-parametric setting will also have infinite VC dimension, it was an open question whether learning is possible in the non-parametric setting under the smoothed online setting. We demonstrate that the nearest neighbor is indeed able to learn in this setting, while generalizing the notion of the smoothed adversary studied by \cite{haghtalab2022smoothed}.

\paragraph{Nearest neighbor methods} The 1-nearest neighbor rule was initially introduced and studied by \cite{fix1951discriminatory}. \cite{cover1967nearest} showed that when the sequence $(x_t)_t$ is drawn i.i.d.\@ from some data distribution $\nu$ over $\mathcal{X}$, nearest neighbor is consistent under the same boundary conditions as our Assumption~\ref{ass:boundary}. There is much work extending the algorithm to other nearest neighbor methods and analyses in the i.i.d.\@ setting \citep{stone1977consistent,devroye1994strong,cerou2006nearest,chaudhuri2014rates}. See also survey work \cite{devroye2013probabilistic,dasgupta2020nearest} and references therein.

There has been limited work on nearest neighbor methods in the non-i.i.d.\@ setting. As noted above, \cite{kulkarni1995rates} studied the online learning setting where the sequence of instances can be arbitrary, but with the Lipschitz constraint on the underlying regression function. While not in the online setting, both \cite{dasgupta2012consistency} and \cite{ben2014domain} considered the consistency of nearest neighbor classifiers where the training and test data distributions differ. In particular, \cite{dasgupta2012consistency} studied nearest neighbor under selective sampling. Here, instances are drawn i.i.d.\@ but only some of the labels are selectively revealed. And \cite{ben2014domain} studied the covariate-shift transfer learning setting where the train $\nu$ and test $\mu$ distributions are related by $\mu(A) < C\nu(A)$ when $A$ comes from some family of measurable sets $\mathcal{A}$.

\newpage
\section{Learning with separation guarantee in the worst-case}
\label{sec:separation-guarantee}

\paragraph{Proof of Proposition~\ref{prop:nonconvergence}}
Suppose that there is a positive separation between classes, so that the margin $m_c(x)$ is lower bounded by some $m> 0$ for all $x \in \mathcal{X}$. By Lemma~\ref{lem:ball-mutually-labeling}, the collection of open balls $B(x, m/3)$ for all $x \in \mathcal{X}$ forms a cover of $\mathcal{X}$ by mutually labeling sets. Because $\mathcal{X}$ is totally bounded, there is a finite subcover of $\mathcal{X}$ by these mutually labeling balls. As each of these sets admits at most one mistake by the \ref{alg:online-nnc} learner, it makes at most finitely many mistakes, achieving sublinear regret.

On the other hand, suppose that there is no positive separation between classes. Then, we can find a sequence of pairs $(x_{2t - 1}, x_{2t})$ such that:
\begin{itemize}
    \item $x_{2t-1}$ is the nearest neighbor of $x_{2t}$ out of all previous instances $x_1,\ldots, x_{2t-1}$, and
    \item $x_{2t-1}$ and $x_{2t}$ are of different classes, $c(x_{2t-1}) \ne c(x_{2t})$.
\end{itemize}
Thus, \ref{alg:online-nnc} makes a mistake $\hat{y}_{2t} = c(x_{2t-1})$ at every even-numbered time, and so:
\[\liminf_{T \to \infty} \frac{1}{T} \sum_{t=1}^T \ind\{y_t \ne \hat{y}_t\} \geq \frac{1}{2}.\]
That is, it fails to achieve sublinear regret.\hfill$\blacksquare$

\newpage
\section{The law of large numbers for martingales} \label{sec:strong-consistency}
For completeness, we include a version of the strong law of large numbers (SLLN).

\begin{theorem}[Strong law of large numbers for martingales, {\citep[Exercise 4.4.11]{durrett2019probability}}] \label{thm:slln}
Let $(M_t)_{t \geq 0}$ be a martingale and let $\xi_t = M_t - M_{t-1}$ for $t > 0$. If $E\xi_t^2 < K < \infty$, then: 
\[\phantom{\mathrm{a.s.}\quad}M_t / t \to 0 \quad \mathrm{a.s.}\]
\end{theorem}

For example, we can use it to formally prove our remark right after Definition~\ref{def:dominated-adversary}, reproduced below. Of course, the remark requires too stringent of a condition to be a useful. But, it is a good demonstration of how to formally define the martingale on which \Cref{thm:slln} can be applied.

\begin{remark}[Example application of the SLLN]
Let $A_t$ be the mistake set of an online learner at time $t$ against a $\nu$-dominated adversary. Suppose that $(A_t)_{t=1}^\infty$ converges to a $\nu$-measure zero set almost surely. Then, the mistake rate converges to zero almost surely as well:
\[\lim_{T \to \infty}\, \frac{1}{T} \sum_{t=1}^T \ind\{y_t \ne \hat{y}_t\} = 0.\]
\end{remark}

\begin{proof}
Define $\{\mathcal{F}_t\}_t$ to be the natural filtration for the stochastic process $\{(x_t, A_{t+1}, \mu_{t+1})\}_t$. The following is a martingale difference sequence: 
\[\xi_{t} = \underbrace{\ind\{x_{t} \in A_{t}\}}_{\ind\{y_{t} \ne \hat{y}_{t}\}} \, - \, \underbrace{E\big[\ind\{x_{t} \in A_{t}\}\,\big|\, \mathcal{F}_{t-1}\big]}_{\mu_{t}(A_{t})}.\]
Since, $E\xi_t^2 < 1$, we can apply \Cref{thm:slln}, which implies:
\[\lim_{T\to \infty} \, \frac{1}{T} \sum_{t=1}^T\big(\ind\{y_{t} \ne \hat{y}_{t}\} - \mu_{t}(A_{t})\big) = 0.\]
Notice that because the adversary is $\nu$-dominated, the almost-sure convergence of $\nu(A_t)$ to zero implies that of $\mu_t(A_t)$ to zero. Thus, the time-averaged expected mistake rate also goes to zero:
\[\lim_{T\to \infty}\, \frac{1}{T} \sum_{t=1}^T \mu_t(A_t) = 0.\]
By dominated convergence, we can sum the previous two equations, proving this remark.
\end{proof}

\newpage 

\section{On the boundary condition} 
Assumption~\ref{ass:boundary} requires that the boundary points to be essentially countable. While this condition does not come for free, counterexamples tend to be fairly pathological. An example where this fails is when the concept is the indicator on the fat Cantor set.

Recall that the fat Cantor set is obtained as the limit of subsets of the unit interval. Each subset looks like a finite union of closed intervals, and at each iteration $n$, the middle $2^{-(n+1)}$-fraction of each interval is removed, leaving behind two smaller closed intervals. The limit is a set with Lebesgue measure 1/2 with no interior; each point in the fat Cantor set is a boundary point. However, all countable sets have Lebesgue measure zero, so the fat Cantor set is not essentially countable.

The Osgood curve gives another counterexample. Recall that a Jordan curve is a closed curve in $\mathbb{R}^2$ that is homeomorphic to the circle, splitting the plane into interior and exterior regions. It is a Jordan curve whose boundary between these two regions has positive measure \citep{osgood1903jordan}.

\newpage

\section{Proofs for rates of convergence} \label{sec:app-metric-measure}
In this section, we provide the background and proofs for \Cref{sec:rates}.

\subsection{Analysis on length spaces}

Recall that length spaces are spaces where distances between points are given by the infimum of lengths over continuous paths between those points. For reference, see also \cite{gromov1999metric}.

\begin{definition}[Length space]
    A metric space $(\mathcal{X},\rho)$ is a \emph{length space} if for all $x, x' \in \mathcal{X}$,
    \[\rho(x,x') = \inf_{\gamma} \, \ell(\gamma),\]
    where $\gamma : [0,1] \to \mathcal{X}$ include all continuous paths from $x$ to $x'$ and $\ell(\gamma)$ is the \emph{length} of the path $\gamma$. 
\end{definition}

\paragraph{Proof of \Cref{lem:margin-length-space}} To show that $m_c(x) = \rho(x, \partial \mathcal{X})$, we prove left and right inequalities.

First, the margin is upper bounded by $m_c(x) \leq \rho(x, \partial \mathcal{X})$. To see this, fix $\delta > 0$. By the definition of the distance between $x$ and the set $\partial \mathcal{X}$, there is a boundary point $z \in \partial \mathcal{X}$ such that:
\[\rho(x, z) < \rho(x, \partial \mathcal{X}) + \frac{\delta}{2}.\]
And as boundary points are arbitrarily close to at least two classes, there exists $x' \in \mathcal{X}$ close to $z$:
\[\rho(z, x') < \frac{\delta}{2},\]
while also belonging to a different class than $x$. By the definition of $m_c(x)$ and by triangle inequality, we obtain that for all $\delta > 0$, there exists some $x'$ satisfying:
\[m_c(x) \leq \rho(x,x') < \rho(x, \partial \mathcal{X}) + \delta.\]
Letting $\delta$ go to zero yields the first inequality.

For the other, we claim that if $\gamma : [0,1] \to \mathcal{X}$ is a continuous path from $x$ to $x'$ with $c(x) \ne c(x')$, then there exists a point $\gamma(t)$ contained in $\partial \mathcal{X}$. If the claim is true, then the other inequality holds:
\[\rho(x, \partial \mathcal{X}) \overset{(i)}{\leq} \inf_{c(x) \ne c(x')}\,\inf_{\gamma}\, \ell(\gamma) \overset{(ii)}{=} \inf_{c(x) \ne c(x')}\, \rho(x,x') \overset{(iii)}{=} m_c(x),\]
where (i) the infimum above is taken over all continuous paths $\gamma$ from $x$ to $x'$, (ii) applies the definition of a length space, and (iii) applies the definition of the margin.

To prove the claim, let $t$ be the first time a point on the path has a different label than $x$. Formally,
\[t := \arginf_{s \in [0,1]}\, \big\{c\big(\gamma(s)\big) \ne c(x)\big\}.\]
To show that $\gamma(t) \in \partial \mathcal{X}$, we need to exhibit a point $\gamma(s)$ that is $\delta$-close to $\gamma(t)$ with a different label, given any $\delta > 0$. Indeed, such a $s$ exists by the definition of $t$ and the continuity of $\gamma$.\hfill $\qed$

\subsection{Analysis on metric measure spaces}
To obtain bounds on the mutually-labeling covering number $\mathcal{N}_\mathrm{ML}(V_r)$, we need to introduce the notion of the \emph{box-counting dimension} a set $A \subset \mathcal{X}$ and the \emph{doubling dimension} of a metric space $\mathcal{X}$. Let us first recall the following definitions and results from analysis and measure theory.

\begin{definition}[Covering number]
    Given $r > 0$ and $A \subset \mathcal{X}$, the \emph{$r$-covering number} $\mathcal{N}_r(A)$ of $A$ is size of a minimal covering of $A$ by balls with radius $r$. 
\end{definition}

\begin{definition}[Box-counting dimension]
    The (upper) \emph{box-counting dimension} of $A \subset \mathcal{X}$ is:
    \[d(A) := \limsup_{r \to 0}\, \frac{\log \mathcal{N}_r(A)}{\log 1/r}.\]
\end{definition}

The box-counting dimension implies a bound on the covering number $\mathcal{N}_r(A)$ of $r^{-d(A) + o(1)}$. The following lemma is a straightforward conversion of the asymptotic limit into a quantitative bound.

\begin{lemma}[Box-counting upper bound on $\mathcal{N}_r$] \label{lem:box-upper-bound}
    Let $\mathcal{X}$ be bounded with diameter $R$. Let $A \subset \mathcal{X}$ have box-counting dimension $d(A)$. Then, for all $c > 0$, there exists a constant $C > 0$ such that:
    \[\mathcal{N}_r(A) < Cr^{-(d(A) + c)}.\]
\end{lemma}
\begin{proof}
    Fix $c > 0$. By the definition of $d(A)$, there exists $r_0 > 0$ such that whenever $0 < r < r_0$,
    \[\frac{\log \mathcal{N}_r(A)}{ \log 1/r} < d(A) + c.\]
    Because $\mathcal{N}_r(A)$ is non-increasing in $r$, we can extend the bound to all $ 0 < r < R$,
    \[ \frac{\log \mathcal{N}_{r} (A)}{ \log 1/(r \wedge r_0)} < d(A) + c,\]
    where $r \wedge r_0 := \min\{r, r_0\}$. In fact, we have $\min\{r, r_0\} > r \cdot r_0 / R$, and so:
    \[\mathcal{N}_r(A) < \left(\frac{r_0}{R} \cdot r\right)^{-(d(A) + c)}.\] 
    To finish the proof, it suffices to let $C = (r_0/R)^{-(d(A) + c)}$.
\end{proof}

\begin{definition}[Doubling dimension] \label{def:doubling-dimension}
    A metric space $(\mathcal{X}, \rho)$ has \emph{doubling dimension} $\Gamma$ if there is a constant $C > 0$ such that for all radius $r > 0$ and centers $x \in \mathcal{X}$, the covering number is bounded:
    \[\mathcal{N}_{r/2}\big(B(x,r)\big) \leq C 2^\Gamma.\]
    We say that $\mathcal{X}$ is \emph{doubling} if it has finite doubling dimension $\Gamma < \infty$.
\end{definition}

To obtain bounds on the mass $\nu(V_r^c)$, we need to introduce the \emph{Minkowski content} of a set $A \subset \mathcal{X}$. First, recall that the $r$-expansion of a set $A$ fattens the set to all points of distance within $r$ of $A$:

\begin{definition}[$r$-expansion]
    Let $A \subset \mathcal{X}$ be a set and $r > 0$. The \emph{$r$-expansion} $A^r$ of $A$ is:
    \[A^r := \bigcup_{x \in A} B(x,r).\]
\end{definition}

The Minkowski content of $A$ is the rate at which an infinitesimal fattening of $A$ increases its mass:

\begin{definition}[Minkowski content]
    Let (upper) \emph{Minkowski content} of $A \subset \mathcal{X}$ is:
    \[\mathfrak{m}(A) := \limsup_{r \to 0} \frac{\nu(A^r) - \nu(A)}{r}.\]
\end{definition}

The following lemma bounding the covering number of the $r$-expansion of a set in terms of the doubling dimension will also be helpful:

\begin{lemma}[Covering the $r$-expansion of a set] \label{lem:covering-r-expansion}
    Let $(\mathcal{X},\rho)$ have finite doubling dimension $\Gamma$. There exists a constant $C > 0$ such that for all $A \subset \mathcal{X}$, we have: 
    \[\mathcal{N}_r(A^r) \leq C2^{\Gamma} \mathcal{N}_r(A).\]
\end{lemma}

\begin{proof}
    Let $A$ be covered by the balls $B(x_1,r),\ldots, B(x_n, r)$ where $n = \mathcal{N}_r(A)$. Then, by the triangle inequality, the $r$-expansion $A_r$ is covered by the $r$-expanded balls, $B(x_1,2r),\ldots, B(x_n, 2r)$. Now, by the definition of the doubling dimension, each expanded ball $B(x_i, 2r)$ can be covered by $C2^{\Gamma}$ balls with radius $r$. It follows that covering $A_r$ needs at most $C2^{\Gamma}n$ balls with radius $r$.
\end{proof}

\subsection{Bounding geometric quantities of $\partial \mathcal{X}$}
\begin{proposition}[Geometric quantities of $\partial \mathcal{X}$] \label{prop:geometric-boundary}
    Let $(\mathcal{X}, \rho, \nu)$ be a bounded length space with finite doubling dimension and Borel measure. Suppose the concept $c$ satisfies $\nu(\partial \mathcal{X}) = 0$. Then, for any $c_1, c_2 > 0$, there is a constant $C > 0$ and $r_0 > 0$ so that for all $ 0 < r < r_0$,
    \[\mathcal{N}_{\mathrm{ML}}\big(V_r\big) \leq C r^{- (d(\partial \mathcal{X}) + c_1)} \qquad \textrm{ and }\qquad \nu\big(V_r^c\big) \leq \big(\mathfrak{m}(\partial \mathcal{X}) + c_2\big) \cdot r.\]
\end{proposition}

\paragraph{Proof of \Cref{prop:geometric-boundary}}
Recall that $V_r$ and $\partial \mathcal{X}$ are defined in terms of the margin:
\[V_r := \{x \in \mathcal{X} : m_c(x) \geq r\}\qquad \textrm{and}\qquad \partial \mathcal{X} = \{x \in \mathcal{X} : m_c(x) = 0\}.\]
While the complement $V_r^c$ always contains the expansion $\partial \mathcal{X}^r$, generally $V_r^c$ can be much larger. But when $\mathcal{X}$ is a length space, equality holds:

\begin{lemma} \label{lem:Vr-rboundary}
    Let $(\mathcal{X}, \rho)$ be a length space. Then, for all $r > 0$:
    \[V_r^c = \partial \mathcal{X}^r.\]
\end{lemma}
\begin{proof}
\Cref{lem:margin-length-space} shows that when $\mathcal{X}$ is a length space, $m_c(x) = \rho(x, \partial \mathcal{X})$. Thus:
\[x \in V_r^c \quad\Longleftrightarrow \quad m_c(x) < r \quad\Longleftrightarrow \quad \rho(x, \partial \mathcal{X}) < r \quad\Longleftrightarrow \quad x \in \partial \mathcal{X}^r.\qedhere\]
\end{proof}
Now, the question of bounding $\mathcal{N}_\mathrm{ML}(V_r)$ and $\nu(V_r^c)$ becomes that of $\mathcal{N}_\mathrm{ML}(\mathcal{X} \setminus \partial \mathcal{X}^r)$ and $\nu(\partial \mathcal{X}^r)$.

\begin{proposition}[Upper bound on $\mathcal{N}_\mathrm{ML}$] \label{prop:ml-n-upper-bound}
    Let $\mathcal{X}$ be a bounded length space with finite doubling dimension $\Gamma$ and diameter $R$. Given a concept $c$, let $d$ be the box-counting dimension of $\partial \mathcal{X}$. Then, for any $c > 0$, there exists a constant $C > 0$ such that for all $r > 0$:
    \[\mathcal{N}_{\mathrm{ML}}(\mathcal{X} \setminus \partial \mathcal{X}^r) \leq CR^{4\Gamma} r^{-(d + c)}.\]
\end{proposition}

\begin{proof}
    We can write $\mathcal{X} \setminus \partial \mathcal{X}^r$ as a union of layers of the form $L_k := \partial \mathcal{X}^{2^{k+1}r} \setminus \partial\mathcal{X}^{2^kr}$,
    \[\mathcal{X} \setminus \mathcal{X}^r = \bigcup_{k=0}^{\lceil \lg R/r\rceil} L_k.\]
    Then, we can upper bound the mutually-labeling covering number by the sum:
    \begin{equation} \label{eqn:ml-cov-num-sum}
    \mathcal{N}_{\mathrm{ML}}\big(\mathcal{X} \setminus \mathcal{X}^r\big) \leq \sum_{k=0}^{\lceil \lg R/r\rceil} \mathcal{N}_{\mathrm{ML}}(L_k).
    \end{equation}
    To upper bound $\mathcal{N}_{\mathrm{ML}}(L_k)$, first note that  by \Cref{lem:Vr-rboundary},
    \[L_k \subset \mathcal{X} \setminus \partial \mathcal{X}^{2^k r} = V_{2^k r}.\]
    Thus, the margin of any point $x \in L_k$ is at least $2^k r$. By \Cref{lem:ball-mutually-labeling}, the ball $B(x, 2^{k}r/3)$ is a mutually-labeling set, so that $\mathcal{N}_{\mathrm{ML}}(L_k) \leq \mathcal{N}_{2^{k} r/3}(L_k)$. In fact, we obtain the following:
    \begin{align}
        \mathcal{N}_{\mathrm{ML}}(L_k) \leq \mathcal{N}_{2^{k} r/3}(L_k) &\overset{(i)}{\leq} \mathcal{N}_{2^{k-2} r}(L_k)  \notag
        \\&\overset{(ii)}{\leq} \mathcal{N}_{2^{k-2} r}(\partial \mathcal{X}^{2^{k+1}r}) \notag
        \\&\overset{(iii)}{\leq} C_1 2^{3\Gamma} \mathcal{N}_{2^{k+1} r}(\partial \mathcal{X}^{2^{k+1}r}) \notag
        \\&\overset{(iv)}{\leq} C_2 2^{4\Gamma} \mathcal{N}_{2^{k+1}r}(\partial \mathcal{X})\notag
        \\&\overset{(v)}{\leq} C_3 2^{4\Gamma} (2^{k+1}r)^{-(d+c)} \label{eqn:ml-lk-upper-bound}
    \end{align}
    where (i) holds because the radius $2^{k-2}r$ is less than $2^k r/3$, (ii) follows because $\partial \mathcal{X}^{2^{k+1}r}$ contains $L_k$ and so has larger covering number, (iii) makes use of the definition doubling dimension three times to convert the $2^{k-2}$-covering number to a $2^{k+1}$-covering number, (iv) applies \Cref{lem:covering-r-expansion} to convert the covering number of the expansion to that of the boundary set, and (v) upper bounds the covering number in terms of the box-dimension of $\partial \mathcal{X}$ by \Cref{lem:box-upper-bound}. 

    By combining \Cref{eqn:ml-cov-num-sum,eqn:ml-lk-upper-bound}, we obtain:
    \[\mathcal{N}_{\mathrm{ML}}\big(\mathcal{X} \setminus \mathcal{X}^r\big) \leq C_3 2^{4\Gamma} r^{-(d + c)} \sum_{k = 0}^\infty 2^{- (d+c)(k+1)},\]
    where the geometric series converges to a constant $2^{-(d+c)}$. We finish by relabeling the constants.
\end{proof}

This shows that $\mathcal{N}(V_r) = r^{-(d(\partial \mathcal{X}) + o(1))}$. Next we show that $\nu(V_r^c) = \big(\mathfrak{m}(\partial \mathcal{X}) + o(1) \big) \cdot r$. This is immediate from the definition of the Minkowski content $\mathfrak{m}(\partial \mathcal{X})$.

\begin{proposition}[Upper bound on $\nu$] \label{prop:nu-upper-bound}
    Let $(\mathcal{X}, \rho, \nu)$ be a metric Borel space. Suppose $c$ is a concept satisfying $\nu(\partial \mathcal{X}) = 0$ whose boundary $\partial \mathcal{X}$ has Minkowski content $\mathfrak{m}$. Then, for any $c > 0$, there exists some $r_0$ such that for all $0 < r < r_0$:
    \[\nu(\partial \mathcal{X}^r) < (\mathfrak{m} + c) \cdot r.\]
\end{proposition}

\begin{proof}
    Since the boundary has measure zero, the definition of Minkowski content states that there exists $r_0 > 0$ so that for all $0 < r < r_0$,
    \[\frac{\nu(\partial \mathcal{X}^r)}{r} < \mathfrak{m}(\partial \mathcal{X}) + c.\]
    The result follows by multiplying through by $r$.
\end{proof}

Together, \Cref{prop:ml-n-upper-bound,prop:nu-upper-bound} prove \Cref{prop:geometric-boundary}.\hfill $\blacksquare$

\subsection{Proofs of convergence rates}
\paragraph{Proof of \Cref{thm:nn-conv-rate}}
Fix $V \subset \mathcal{X}$. Let $(x_t)_t$ be the sequence of test instances. Denote by $A_t \subset \mathcal{X}$ region on which \ref{alg:online-nnc} makes a mistake at time $t$. We can count the total number of mistakes separately on $V$ and $V^c$:
\begin{align*} 
\#\mathrm{mistakes}_T &:= \sum_{t=1}^T \ind\{x_t \in A_t\} =  \sum_{t=1}^T \underbrace{\ind\{x_t \in A_t \cap V\}}_{\textrm{mistakes made in $V$}} + \sum_{t=1}^T \underbrace{\ind\{x_t \in A_t \cap V^c\}}_{\textrm{mistakes made in $V^c$}}
\end{align*}
Because at most one mistake can be made per mutually-labeling set on $V$, the first summation can be bounded by $\mathcal{N}_\mathrm{ML}(V)$. The second term can be bounded by the number of times $x_t$ comes from $V^c$:
\begin{align*} \ind\{x_t \in A_t \cap V^c\} &\leq \ind\{x_t \in V^c\}
\\&= \underbrace{\ind\{x_t \in V^c\} - \mu_t(V^c)}_{\textrm{martingale difference}} + \mu_t(V^c).
\end{align*}
By Azuma-Hoeffding's, we have that with probability at least $1 - p/2T^2$:
\[\sum_{t=1}^T \underbrace{\ind\{x_t \in V^c\} - \mu_t(V^c)}_{\textrm{martingale difference}} \leq \sqrt{T \log \frac{2T^2}{p}} \leq \sqrt{2 T \log \frac{2T}{p}}.\]
Because the adversary is $\nu$-dominated, we also have $\mu_t(V^c) < \epsilon\big(\nu(V^c)\big)$. By taking a union bound over all $T \in \mathbb{N}$, we obtain that with probability at least $1 - p$,
\[\#\mathrm{mistakes}_T \leq \mathcal{N}_\mathrm{ML}(V) + T \epsilon\big(\nu(V^c)\big) + \sqrt{2T \log \frac{2T}{p}}.\]
The result follows from optimizing $V$, and by noting at most $T$ mistakes can be made in $T$ time.
\hfill $\blacksquare$

\paragraph{Proof of \Cref{thm:conv-nn-sigma}}
Given $c_1, c_2 > 0$, \Cref{prop:geometric-boundary} yields $C, r_0 > 0$ so that when $0 < r < r_0$,
\[\mathcal{N}_{\mathrm{ML}}\big(V_r\big) \leq C r^{- (d + c_1)} \qquad \textrm{ and }\qquad \nu\big(V_r^c\big) \leq \big(\mathfrak{m} + c_2\big) \cdot r.\]
From \Cref{thm:nn-conv-rate}, it follows that with probability at least $1 - p$, we have for all $T$:
\begin{align*} 
\# \mathrm{mistakes}_T &\leq \inf_{0 < r < r_0}\, \mathcal{N}_{\mathrm{ML}}(V_r) + T \epsilon\big(\nu(V_r^c)\big) + \sqrt{2T \log \frac{2T}{p}}
\\&\leq \inf_{0 < r < r_0}\, C r^{- (d + c_1)} + T \sigma^{-1} \cdot \big(\mathfrak{m} + c_2\big) \cdot r + \sqrt{2T \log \frac{2T}{p}},
\end{align*}
where $r$ is optimized at:
\[r_T^* = \left(\frac{C (d + c_1) \sigma}{T (\mathfrak{m} + c_2)}\right)^{1 / (d + c_1 + 1)},\]
provided that $r_T^* < r_0$. This will eventually hold for sufficiently large $T > T_0$. For $T \leq T_0$, we can use the coarser mistake bound $T_0$. Thus, for all $T \in \mathbb{N}$:
\[\#\mathrm{mistakes}_T \leq T_0 + C_1 \left(\frac{T\big(\mathfrak{m} + c_2\big)}{\sigma}\right)^{(d + c_1)/(d + c_1 + 1)} + \sqrt{2T \log \frac{2T}{p}},\]
where $C_1$ is a constant, defined below. 

Because we assumed $d > 1$, the $\sqrt{T \log T}$ term is eventually dominated by the $T^{(d + o(1))/(d + 1)}$ term when $T > T_0'$ is sufficiently large. We obtain the result by setting $C_0$ as below, and noting that we can simplify the exponent because $(d + c_1)/ (d + c_1 + 1) < (d + c_1)/(d + 1)$.
\begin{itemize}
    \item $C_0 = T_0 + 2\sqrt{2 T_0' \log \frac{2T_0'}{p}}$.
    \item $C_1 = 2C(d + c_1)$.
\end{itemize}
\hfill $\blacksquare$

\end{document}